\documentclass[nonacm,pdftex]{acmart}

\usepackage{lipsum}
\usepackage{amsmath}

\usepackage{natbib} 
\usepackage{mathtools} 
\usepackage{siunitx} 
\usepackage{booktabs} 
\usepackage{tikz} 

\usepackage{amsmath}
\usepackage{mathtools}
\usepackage{amsthm}
\usepackage{mathrsfs}
\usepackage{bm}
\usepackage{xcolor}
\usepackage{subcaption}
\usepackage{longtable}
\usepackage{rotating}
\usepackage{soul}
\usepackage{floatpag}
\usepackage{multirow}
\usepackage[capitalize,noabbrev]{cleveref}

\def\featspace{{\mathcal{X}}}
\def\labelspace{{\mathcal{Y}}}
\def\hypospace{{\mathcal{H}}}

\def\eqref#1{equation~\ref{#1}}

\def\1{\bm{1}}

\def\vx{{\bm{x}}}

\DeclareMathAlphabet{\mathsfit}{\encodingdefault}{\sfdefault}{m}{sl}
\SetMathAlphabet{\mathsfit}{bold}{\encodingdefault}{\sfdefault}{bx}{n}

\DeclareMathOperator*{\argmin}{arg\,min}

\definecolor{color1}{HTML}{b26062}
\definecolor{color2}{HTML}{c88a67}
\definecolor{color3}{HTML}{619c6f}
\DeclareRobustCommand{\legendsquare}[1]{%
  \textcolor{#1}{\rule{2ex}{1.5ex}}%
}

\author{Shlok Mehendale}
\affiliation{
  \institution{CSIS, BITS Pilani KK Birla Goa Campus}
  \country{India}
}

\author{Aditya Challa}
\affiliation{
  \institution{CSIS, BITS Pilani KK Birla Goa Campus}
  \country{India}
}

\author{Rahul Yedida}
\affiliation{
  \institution{Lexis Nexis}
  \city{North Carolina}
  \country{USA}
}

\author{Sravan Danda}
\affiliation{
  \institution{CSIS, BITS Pilani KK Birla Goa Campus}
  \country{India}
}

\author{Santonu Sarkar}
\affiliation{
  \institution{CSIS, BITS Pilani KK Birla Goa Campus}
  \country{India}
}

\author{Snehanshu Saha}
\affiliation{
  \institution{CSIS, BITS Pilani KK Birla Goa Campus}
  \country{India}
}

\def\rnthm{{Radon--Nikod\'{y}m}}

\begin{document}

\title{A \rnthm{} Perspective on Anomaly Detection: Theory and Implications}

\begin{abstract}

Which principle underpins the design of an effective anomaly detection loss function? The answer lies in the concept of \rnthm{} theorem, a fundamental concept in measure theory. The key insight from this article is -- Multiplying the vanilla loss function with the \rnthm{} derivative improves the performance across the board. We refer to this as RN-Loss. We prove this using the setting of PAC (Probably Approximately Correct) learnability. 

Depending on the context a \rnthm{} derivative takes different forms. In the simplest case of supervised anomaly detection, \rnthm{} derivative takes the form of a simple weighted loss. In the case of unsupervised anomaly detection (with distributional assumptions), \rnthm{} derivative takes the form of the popular cluster based local outlier factor.

We evaluate our algorithm on 96 datasets, including univariate and multivariate data from diverse domains, including healthcare, cybersecurity, and finance. We show that RN-Derivative algorithms  outperform state-of-the-art methods on 68\% of Multivariate datasets (based on F-1 scores) and also achieves peak F1-scores on 72\% of time series (Univariate) datasets. 

\end{abstract}

\maketitle

\section{Introduction}
\label{sec:intro}

Anomaly detection is the process of identifying rare yet significant deviations from normal patterns. This has become essential in various domains such as finance, healthcare, and cyber-security, where undetected anomalies can lead to catastrophic consequences. Moreover, when detected, anomalies can provide significant value. Despite its practical importance, the diversity of real-world settings hinders a unified theoretical treatment. Few aspects which affect the good theoretical framework are
\begin{enumerate}
    \item[(a)] Supervised vs Unsupervised Anomaly Detection: One may or may not have access to sample anomalies in practical settings. If a sample of \emph{labeled anomalies} is provided, it is referred to as supervised anomaly detection. Else it is referred to as unsupervised anomaly detection. \textbf{Remark:} An important subtlety here is that -- One might have anomalies in the sample but they are not labeled. This is also conventionally categorized as unsupervised anomaly detection. In some settings, a small number of labeled anomalies (or partially labeled data) are available — this is often referred to as semi-supervised or weakly supervised anomaly detection (\cite{PReNet},\cite{DeepSAD})
    \item[(b)] Percentage of Anomalies in the train set: It is intuitively clear that if the percentage of anomalies in the train set is higher, it makes anomaly detection simpler. However, different algorithms are usually preferred at different thresholds of anomalies. 
    \item[(c)] Dimensionality of the dataset:  As with most of machine learning algorithms, dimension plays an important role in the ability to find the right hypothesis function. Simple univariate anomaly detection can be performed using classical statistical measures such as standard deviations and/or quantiles. Multivariate anomaly detection is significantly harder.
    \item[(d)] Time Series or not: Time series anomaly detection adds an additional layer of complexity. Specifically evaluation is impacted in this setting. 
\end{enumerate}

\paragraph{Motivation of the present work:} To our knowledge, there is no unifying principle which spans all the above mentioned aspects in anomaly detection. The aim of this article is to provide a simple foundational principle which can assist in all the settings above. Specifically, using the PAC learning framework we propose a simple principle -- \textbf{Multiplication with the \rnthm{} derivative improves the performance}. And since PAC learning is a broad framework, this insight spans all the above mentioned aspects.

\subsection{Contributions}

The key contribution of this paper is an overarching technique for anomaly detection irrespective of the type of supervision, frequency of anomalies, size of the dataset or dimensionality. The technique is based on an elegant yet simple mathematical framework which connects both supervised and unsupervised anomaly detection paradigms without requiring changes in model architecture or optimization procedures. It is important to observe that the elegance lies in correcting the distributional differences between the training and evaluation distributions (Details are provided in \cref{sec:3}).\\  
\textit{Theoretical contributions:} We
\begin{enumerate}
    \item[(A)] Introduce a weighted loss function based on the \rnthm{} derivative (termed “RN-Loss” in this paper), tailored for supervised anomaly detection. This implies, one can estimate/approximate the \rnthm{} derivative using class dependent weights, resulting in the weighted loss function.
    \item[(B)] Show that, in the context of unsupervised anomaly detection, popular and time-tested algorithms such as cluster based local outlier factor (CBLOF) and its variants can be derived using \rnthm{} derivative based correction. This correction enables using the same framework in unsupervised anomaly detection.
    \item[(C)] Introduce the problem of (PAC-)learnability of anomaly detection problem in \cref{sec:2}. 
    \item[(D)] Show that anomaly detection is indeed PAC learnable in \cref{sec:3}. 
\end{enumerate}
The theoretical contributions and derivation provide the foundational principles and significant practical insights. We state them below and leverage these crucial facts throughout the remainder of the article.
\begin{itemize}
    \item Product of the original loss function and the \rnthm{} derivative improves the loss function. This is theoretically demonstrated in \cref{sec:31} and empirically validated in \cref{sec:4}.  
    \item \rnthm{} derivative offers a mathematically grounded abstraction for designing the loss functions for anomaly detection. (See \cref{sec:3})
    \item Depending on the context (supervised/unsupervised), the \rnthm{} derivative can take different forms. (Derivations in \cref{sec:32})
\end{itemize}

\paragraph{Empirical Contributions:} The proposed RN-Loss framework has several practical advantages
\begin{enumerate}
    \item[1.] RN-Loss maintains \emph{computational efficiency} by building on base loss functions like Binary Cross-Entropy,
    \item[2.] RN-Loss can be readily incorporated into existing training pipelines, requiring no changes to model architectures or optimization procedures.
    \item[3.] Unsupervised methods like dBTAI~\cite{dBTAI} benefit from using a modified version of RN-Loss. Using RN-Loss makes it capable of identifying anomalies even when the model is trained solely on normal data. 
    \item[4.] The loss function also demonstrates \emph{flexibility}, fitting varied data distributions such as Weibull and Log-normal without requiring structural changes. 
\end{enumerate}
In summary, these properties make RN-Loss a robust and adaptable solution for anomaly detection, offering improved performance metrics, computational efficiency, and versatility across a wide range of real-world applications.

\paragraph{Empirical Performance:} The RN-Loss function delivers significant improvements in anomaly detection, offering both \emph{enhanced performance} and \emph{broad adaptability}. It surpasses prior state-of-the-art (SoTA) methods, improving F1 scores on 68\% and Recall on 46\% of the multivariate datasets, with similar trends observed in univariate time-series data (F1: 72\%, Recall: 83\%). These results highlight its \emph{consistent effectiveness across diverse benchmarks}.

Our experiments further demonstrate that RN-Loss substantially enhances the performance of unsupervised anomaly detection methods—\emph{specifically the vanilla implementations of CBLOF and ECBLOF \cite{ECBLOF} when integrated with clustering algorithms such as K-Means \cite{CBLOF} and dBTAI \cite{dBTAI}}. The enhanced KMeans-CBLOF configuration achieved superior results on 93\% of univariate datasets (27 out of 29) and on 48\% of multivariate datasets (32 out of 67), relative to the original version. Although dBTAI previously achieved SoTA performance, its evaluation metrics—particularly precision—were inflated. By incorporating RN-Loss, these metrics were better calibrated, and the modified dBTAI maintained or improved overall performance across 59 multivariate datasets, while showing increased recall on nearly all univariate datasets. \cref{fig:summary} provides a visual summary the results.

\begin{figure}[t]
    \centering
    \includegraphics[width=\linewidth]{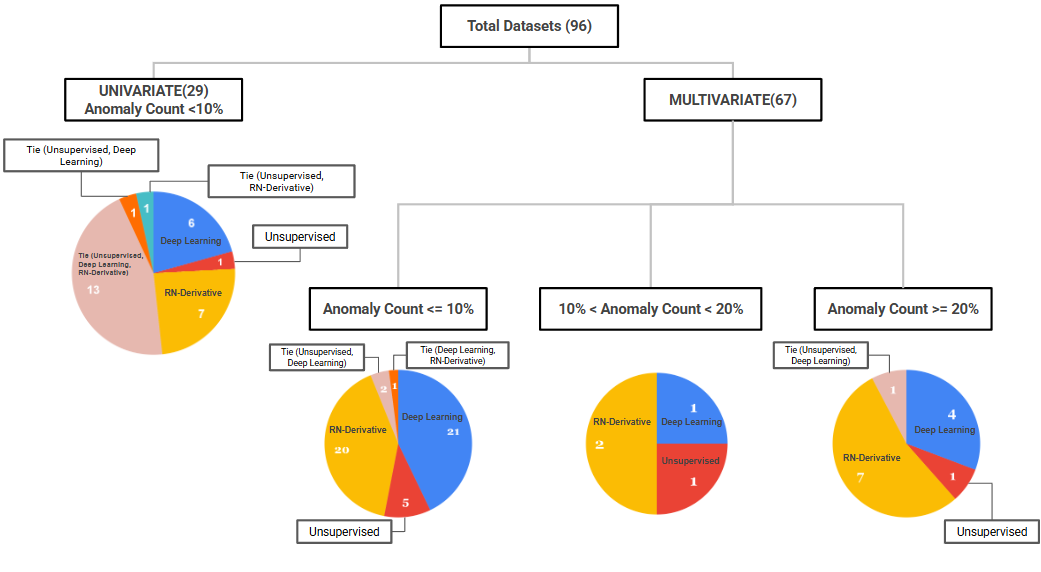}
    \Description{A bar chart}
    \caption{Comparative Analysis of Anomaly Detection Algorithms. Performance evaluation prioritizes recall, with precision as the secondary metric for tied cases. The comparison spans three algorithm categories: (1) Deep Learning approaches (AutoEncoders, DAGMM, DevNet, GAN, DeepSAD, FTTransformer (current state-of-the-art), and PReNet), (2) Unsupervised methods (LOF, Elliptic Envelope, Isolation Forest, dBTAI, MGBTAI, and quantile-based approaches including q-LSTM variants and QReg), and (3) RN-Derivative algorithms (RN-Net and RN-LSTM). RN-Net outperforms state-of-the-art methods on 68\% of Multivariate datasets (based on F-1 scores), while the RN-LSTM + RN-Net combination achieves peak F1-scores on 72\% of time series (Univariate) datasets. For detailed numerical comparisons}
    \label{fig:summary}
\end{figure}

\section{Problem Setup and Notation}
\label{sec:2}

Let $\featspace \subset \mathbb{R}^d$ denote the feature space. We assume that the sample is obtained from a \emph{mixture of inline and anomaly distributions}. Let $D_I$ denote the inline distribution and $D_{A}$ denote the anomaly distribution. If the \emph{contamination ratio} or \emph{anomalous ratio} is given by $\alpha$, the sample is obtained from
\begin{equation}
    D^{\alpha} = (1-\alpha) D_I + \alpha D_A
\end{equation}
Note that support of both $D_I$ and $D_A$ distributions is assumed to be $\featspace$. As stated before, anomaly detection can either be supervised or unsupervised.

\paragraph{Supervised Anomaly Detection:} In the supervised case, we assume to also have access to labels, $\labelspace = \{0, 1\}$ where $1$ indicates that the sample belongs to the inline distribution $D_I$ and $0$ indicates that the sample belongs to the anomaly distribution $D_A$. 

We denote the joint distribution of $n$ i.i.d samples from a given distribution $D$ using $D^n$. If $S \coloneqq \{(\vx^i, y^i)\}$ denotes the sample of size $n$ drawn i.i.d from $D^{\alpha}$, this is assumed to have the joint distribution $D^{\alpha,n}$.

Given a sample of points $S \sim D^{\alpha,n}$ the aim is to obtain a classifier $f$ such that, for any sample $\vx$ from $D^{\alpha}$ : (i) if $\vx$ is sampled from $D_{A}$ then identify it as \emph{anomaly} and (ii) if $\vx$ is sampled from $D_{I}$ identify it as \emph{non-anomaly} or \emph{inline}. Note that this becomes a binary classification problem

\paragraph{Unsupervised Anomaly Detection:} Given a sample of points $S \coloneqq \{(\vx^i)\}$ drawn i.i.d from $D^{\alpha}$, the aim is to obtain a classifier $f$ such that, for any sample $\vx$ : (i) if $\vx$ is sampled from $D_{A}$ then identify it as \emph{anomaly} and (ii) if $\vx$ is sampled from $D_{I}$ identify it as \emph{non-anomaly}. \textbf{Remark:} For the sake of generality, we assume that the sample is obtained from the contaminated distribution $D^{\alpha}$ instead of the inline distribution $D_I$.  

\paragraph{Space of Distributions:} If there is no restriction on the set of possible distributions $D_{XY}$, no-free-lunch theorem \cite{Wolpert} suggests that anomaly detection is impossible. So, we restrict the set of distributions to a set $\mathscr{D}_{X}$.

\paragraph{Hypothesis Space, Loss function and Risks :} Let $\hypospace \subset \{h : \featspace \to \{0,1\}\}$ denote the set of functions from which we choose our classifier $f$. Any specific $h \in \hypospace$ is referred to as hypothesis function or simply function if the context is clear. We consider 0-1 loss function -- $\ell : \labelspace \times \labelspace \to \{0,1\}$ where $\ell(y_1, y_2) = 0$ if $y_1 = y_2$ and $\ell(y_1, y_2) = 1$ otherwise.

The risk associated with a hypothesis is defined as the expected loss when the hypothesis is used for identifying anomalies for a given distribution $D \in \mathscr{D}_{X}$  
\begin{equation}
    R_{D}(h) = E_{D}[\ell(h(\vx), y)]
\end{equation}

\subsection{PAC Learning Framework}

We use the following definition of PAC learning \cite{Ratsaby2008} in this article -- 

\paragraph{PAC Learning:} A hypothesis class $\hypospace$ is PAC learnable if there exists (i) an algorithm $\mathbf{A} : \cup_{n=1}^{\infty} (\featspace \times \labelspace)^n \to \hypospace$, (ii) a decreasing sequence $\epsilon(n) \to 0$, and (iii) for all $D \in \mathscr{D}_{X}$
\begin{equation}
    E_{S \sim D^{n}} [R_{D}(\mathbf{A}(S)) - \inf_{h \in \hypospace}R_{D}(h)] \leq \epsilon(n)
    \label{eq:standardpac}
\end{equation}
Existing results \cite{Uniconv} show that if $\hypospace$ has finite VC dimension, then it is learnable for all possible distributions $\mathscr{D}_{XY}$. Specifically, neural networks are PAC learnable.

\begin{quote}
\it
    An important subtlety in the PAC learning framework is the implicit assumption that the train and evaluation distributions are the same. However, this is not true for anomaly detection.     
\end{quote}

\paragraph{PAC Learning for Anomaly Detection:} Define $\alpha-$risk to quantify the expected loss when \emph{anomaly contamination ratio} is $\alpha \in [0,1)$. That is, $D^{\alpha} = (1-\alpha) D_I + \alpha D_A$
\begin{equation}
    R_{D}^{\alpha}(h) = (1-\alpha) R_{D_I}(h) + \alpha R_{D_A}(h)
\end{equation}

A hypothesis class $\hypospace$ is PAC learnable for Anomaly Detection if there exists (i) an algorithm $\mathbf{A} : \cup_{n=1}^{\infty} (\featspace \times \labelspace)^n \to \hypospace$, (ii) a decreasing sequence $\epsilon(n) \to 0$, and (iii) for all $D_{XY} \in \mathscr{D}_{XY}$
\begin{equation}
    E_{S \sim D^{\alpha, n}} [R_{D}^{0.5}(\mathbf{A}(S)) - \inf_{h \in \hypospace} R_{D}^{0.5}(h)] \leq \epsilon(n) 
    \label{eq:anomalypac}
\end{equation}

\textbf{Important Remark:} In words, while learning is done using the samples are obtained from specific $\alpha$, the evaluation is fixed at $\alpha=0.5$. This is because, in practice one is interested in the metrics associated with anomaly class. Thus, convergence with respect to the risks $R_{D}^{\alpha}$ is not consistent with anomaly detection. Instead, one has to consider the risks $R_{D}^{0.5}$. As we shall shortly see, this observation allows us to correct the loss function easily using \rnthm{} derivative.

\section{RN-Loss : Derivation from first principles}
\label{sec:3}

In this section, we 
\begin{enumerate}
    \item[1.] Answer the question -- \emph{When is a hypothesis class $\hypospace$ PAC Learnable for Anomaly Detection?}
    \item[2.] Propose RN-Loss, which is a generic way to design loss functions for anomaly detection. We illustrate this by applying the RN-Loss principle in both supervised and unsupervised context.
\end{enumerate}

\subsection{PAC Learnable for Anomaly Detection}
\label{sec:31}

Recall the \rnthm{} theorem from measure theory.
\begin{theorem}[\rnthm{} \cite{RN_Thm}]
\label{thm:1}
    Let $(\Omega,\mathcal{A})$ be a measurable space with $\mathcal{A}$ as the $\sigma$ algebra and $\mu, \nu$ denote two $\sigma-$finite measures such that $\nu << \mu$ ($\nu$ is absolutely continuous with respect to $\mu$). Then, there exists a function $f$ such that,
    \begin{equation}
        \nu(A) = \int_{A} f d \mu \quad (or) \quad d \nu = f d \mu
    \end{equation}
    where A $\in \mathcal{A}$.
\end{theorem}

\paragraph{Assumption 1 (Absolute Continuity):} For any fixed $\alpha \in (0,1)$, let $\nu$ denote the measure induced by $D^{0.5}$  and $\mu$ denote the measure induced by $D^{\alpha}$. We assume that \emph{$\nu$ is absolutely continuous with respect to $\mu$} ( $\nu << \mu$). This is a reasonable practical assumption since both $D^{0.5}$ and $D^{\alpha}$ are a mixture of same distributions $D_{I}, D_{A}$.

\paragraph{Assumption 2 (Boundedness):} From Assumption 1 and \cref{thm:1} we have that there exists a \rnthm{} derivative $f$ relating $\nu, \mu$ as -- $d \nu = f d \mu$. We further assume that this is bounded, i.e $1/b < f < b$  for some $b$ on the support of $\mu$. This is a reasonable assumption as well since both $\mu, \nu$ represent mixtures of the same underlying distributions.

We then have the following proposition:

\begin{proposition}
\label{thm:2}
    Let $\nu$ denote the measure induced by $D_{XY}^{0.5}$  and $\mu$ denote the measure induced by $D_{XY}^{\alpha}$. Also, let $d \nu = f d \mu$ (absolutely continuous) where $1/b < f < b$ for some $b < \infty$ on the support of $\mu$. For all $h \in \hypospace$, there exists a $\Delta_{\mu, \nu}$ such that,
    \begin{equation}
        \frac{1}{\Delta_{\mu, \nu}} \leq \frac{R_{D}^{0.5}(h)}{R_{D}^{\alpha}(h)} \leq \Delta_{\mu, \nu}
    \end{equation}
\end{proposition}

\begin{proof}
Recall, that 
\begin{equation}
    R_{D}^{0.5}(h) = \int \ell(h(\vx),y) d\nu = \int \ell(h(\vx),y) f d\mu \quad and \quad R_{D}^{\alpha} = \int \ell(h(\vx),y) d\mu
    \label{eq:8}
\end{equation}
So, we have
\begin{equation}
    \frac{R_{D}^{0.5}(h)}{R_{D}^{\alpha}(h)} = \frac{\int \ell(h(\vx),y)f d\mu}{\int \ell(h(\vx),y) d\mu}
\end{equation}
Now, observe that $\ell$ is taken to be 0-1 loss, $f$ is bounded by $1/b$ and $b$ on the support of $\mu$, and $\int d\mu = 1$ (probability measure). Note that the bound $b$ depends on the distributions $\mu, \nu$. Hence, we have 
\begin{equation}
    \frac{1}{\Delta_{\mu, \nu}} \leq \frac{R_{D}^{0.5}(h)}{R_{D}^{\alpha}(h)} \leq \Delta_{\mu, \nu}
\end{equation}
for some constant $\Delta_{\mu, \nu}$
\end{proof}

\begin{theorem}
    If $\hypospace$ is PAC Learnable, then $\hypospace$ is PAC Learnable for Anomaly Detection.
\end{theorem}

\begin{proof}
    Assume that $D^{\alpha} \in \mathscr{D}_{X}$ for all $\alpha \in (0,1)$. If $\hypospace$ is PAC learnable, we have that there exists
    \begin{enumerate}
        \item[(i)] an algorithm $\mathbf{A} : \cup_{n=1}^{\infty} (\featspace \times \labelspace)^n \to \hypospace$
        \item[(ii)] a decreasing sequence $\epsilon(n) \to 0$ such that,
    \end{enumerate}
    for all $D \in \mathscr{D}_{X}$
    \begin{equation}
        E_{S \sim D^n}[R_{D}(\mathbf{A}(S)) - \inf_{h \in \hypospace} R_{D}(h)] \leq \epsilon(n)
    \end{equation}

    Now, from \cref{thm:2}, we have
    \begin{equation}
        R_{D}^{0.5}(\mathbf{A}(S))  \leq R_{D}^{\alpha}(\mathbf{A}(S)) \times \Delta_{\mu, \nu}
    \end{equation}
    and,
    \begin{equation}
            \inf_{h \in \hypospace} R_{D}^{0.5}(h) \geq \frac{1}{\Delta_{\mu,\nu}} \inf_{h \in \hypospace} R_{D}^{\alpha}(h) 
    \end{equation}
    Hence,
    \begin{equation}
    \begin{aligned}
        R_{D}^{0.5}(\mathbf{A}(S)) - &\inf_{h \in \hypospace} R_{D}^{0.5}(h)
         \leq R_{D}^{\alpha}(\mathbf{A}(S)) \times \Delta_{\mu, \nu} - \frac{1}{\Delta_{\mu,\nu}} \inf_{h \in \hypospace} R_{D}^{\alpha}(h) \\
        &\leq \Delta_{\mu, \nu} \left[ R_{D}^{\alpha}(\mathbf{A}(S)) -  \inf_{h \in \hypospace} R_{D}^{\alpha}(h) \right] 
         + \inf_{h \in \hypospace} R_{D}^{\alpha}(h) \left[ - \frac{1}{\Delta_{\mu,\nu}} +  \Delta_{\mu, \nu}\right]
    \end{aligned}    
    \end{equation}
    Taking expectations on both sides,
    \begin{equation}
    \begin{aligned}
        E_{S \sim D^{\alpha, n}} [R_{D}^{0.5}(\mathbf{A}(S)) - \inf_{h \in \hypospace} R_{D}^{0.5}(h)] 
        & \leq \Delta_{\mu, \nu} E_{S \sim D^{\alpha, n}} [R_{D}^{\alpha}(\mathbf{A}(S)) - \inf_{h \in \hypospace} R_{D}^{\alpha}(h)] \\
        & + \left[ - \frac{1}{\Delta_{\mu,\nu}} +  \Delta_{\mu, \nu}\right] E_{S \sim D^{\alpha, n}} \left[\inf_{h \in \hypospace} R_{D}^{\alpha}(h)\right]
    \end{aligned}    
    \label{eq:15}
    \end{equation}
    Using the \underline{Realizability Assumption} of the PAC learning framework, we have that
    \begin{equation}
        E_{S \sim D^{\alpha, n}} \left[\inf_{h \in \hypospace} R_{D}^{\alpha}(h)\right] = 0.
    \end{equation}
    Thus, we have that for the same algorithm $\mathbf{A}(S)$ and sequence $\epsilon(n)$,
    \begin{equation}
    E_{S \sim D^{\alpha, n}} [R_{D}^{0.5}(\mathbf{A}(S)) - \inf_{h \in \hypospace} R_{D}^{0.5}(h)] \leq \epsilon(n)
    \end{equation}

    Hence $\hypospace$ is PAC learnable for anomaly detection as well. 
\end{proof}

The above proof offers some significant insights into learning for Anomaly Detection:
\begin{itemize}
    \item[1.] The algorithm $\mathbf{A}$ does not have to change significantly. The only change required is to modify the algorithm to suit \cref{eq:anomalypac} instead of \cref{eq:standardpac}.
    \item[2.] The modification of $\mathbf{A}$ depends on the \rnthm{} derivative. This can be observed from \cref{eq:15}, where the constant $\Delta_{\mu, \nu}$ plays a key role in ensuring convergence. \cref{thm:2} and \cref{eq:8} shows that $\Delta_{\mu, \nu}$ is the constant depending on \rnthm{} derivative.
    \item[3.] The hypothesis class $\hypospace$ should be large enough to accommodate all possible distributions $D^{\alpha}$. However, thanks to recent advances in machine learning, this is not a strong practical restriction. Standard classes such as neural networks and even boosted trees satisfy this assumption. 
    \item[4.] \emph{Impact of Realizability Assumption:} In the standard PAC learning where the realizability assumption can be easily overcome by considering the bayes optimal classifier. However, for anomaly detection  realizability assumption becomes crucial for ensuring learnability. This can be observed from \cref{eq:15}. Nevertheless, this is not a strong restriction in practice. 
\end{itemize}

\subsection{Estimating \rnthm{} derivatives in different contexts}
\label{sec:32}

We now dive deeper into how the algorithm $\mathbf{A}$ must be adapted for anomaly detection. The \underline{overarching principle} is to estimate the \rnthm{} derivative $f \approx \hat{f}$, so that the loss function $\ell(h(x),y)$ is transformed to $\ell(h(x),y)\hat{f}$. We refer to this as \textbf{RN-Loss}. 

\subsubsection{Case 1: Supervised Anomaly Detection} Recall, $D_I$ and $D_A$ are two probability distributions on a measurable space $\featspace \times \labelspace$. Denote their respective densities by $p_I(\vx,y)$ and  $p_A(\vx,y)$. Recall, for $\alpha \in [0,1]$, $D^{\alpha}$ is given by $(1-\alpha) D_I + \alpha D_A$. Denote by $\mu$ the probability measure induced by $D^{\alpha}$. Equivalently, $\mu$ has density
\begin{equation}
  \mu(\vx,y) = (1-\alpha) p_I(\vx,y) + \alpha p_A(\vx,y).
\end{equation}
We also consider distribution $D^{0.5}$, whose induced measure is denoted $\nu$, with density
\begin{equation}
  \nu(\vx,y) = 0.5 p_I(\vx,y) + 0.5 p_A(\vx,y)
\end{equation}
$\nu$ is absolutely continuous with respect to $\mu$ $\nu << \mu$. By the \rnthm{} theorem, there is a function $f(\vx,y)$ such that
\begin{equation}
   d \nu = f d \mu \quad \text{or} \quad  \frac{d\nu}{d\mu}(\vx,y) = f(\vx,y).
\end{equation}
Under the usual condition that $\mu(\vx,y) > 0$, we obtain
\begin{equation}
  f(x,y) = \frac{\nu(\vx,y)}{\mu(\vx,y)} =\frac{0.5p_I(\vx,y) + 0.5p_A(\vx,y)}
       {(1-\alpha)p_I(\vx,y) + \alpha p_A(\vx,y)}.
  \label{eq:RNderiv}
\end{equation}
\emph{The key idea is to use the empirical distribution function to estimate $f$}. Given a sample $\{\vx^i, y^i\}$ from $D^{\alpha}$ where $y^{i} = 1$ if it belongs to $D_I$ and $y^{i} = 0$ if it belongs to $D_A$. From above we have,
\begin{equation}
    f(\vx, +1) = \frac{0.5}{1-\alpha} \quad and \quad f(\vx, 0) = \frac{0.5}{\alpha}
\end{equation}

\textbf{Important Remark:} Interestingly, due to the setup we have that the weights $0.5/(1-\alpha), 0.5/\alpha$ \underline{does not depend on} $\vx$. This is because the \rnthm{} derivative includes the \emph{same} distribution both in the numerator and denominator in each of the cases. 

Since, constant multiples do not effect the optimization, we make another simplification as follows -- Let the weight of the anomalous class $D_A$ to be $1$, samples from the inline class $D_I$ are then reweighted by
\begin{equation}
    \omega = \frac{\alpha}{1-\alpha} \approx \frac{\# \text{Anomalies}}{\# \text{Inline Samples}}
\end{equation}

Thus, in the case of supervised anomaly detection one only needs to adjust the weights of the samples when learning the hypothesis $h$.

\subsubsection{Case 2: Unsupervised Anomaly Detection}

In the case of unsupervised anomaly detection, one does not have access to sample anomalies. In such cases, one makes additional assumptions such as gaussian distribution to obtain the hypothesis function. The most popular assumption is that the sample is obtained from \emph{mixture of gaussian}.  Here, we illustrate how the \rnthm{} derivative correction is applied.

Let $\featspace \subset \mathbb{R}^d$ be the data space. Given a sample $S = \{ \vx_i \}_{i=1}^{n}$ drawn from the distribution $D^{\alpha}$, we aim to assign an anomaly score to each $x \in \featspace$. Moreover, since we do not have a explicit loss function, one uses a heuristic $h(x)$ to assign scores. In this case we have $\ell(h(\vx),y) \approx h(x)$ -- That is we replace the loss itself with the heuristic scores. Accordingly following our \rnthm{} derivative principle we correct the heuristic $h(x)$ with the \rnthm{} derivative $f(x)$ to obtain $h(x)f(x)$ as the final scoring rule.

\paragraph{Hypothesis Function $h^*$:} Any point can be considered an anomaly if it lies far-away from the center. Thus, under the assumption of mixture-of-gaussians, we have the following anomaly score
\begin{equation}
    h^*(x) = \argmin_i \| x - \mu_i \|
    \label{eq:a31}
\end{equation}
where $\mu_i$ denotes the centroid of the $i$th Gaussian. In words, we consider the distance of the point from the nearest cluster.

\paragraph{Distribution assumption of $h^*$:} Let $\{C_1, C_2, \cdots, C_m\}$ denote the clusters and $p_{C_i}(x)$ denote the distribution of each cluster $C_i$. We have the following implicit \emph{distribution assumption}
\begin{equation}
    D^{\alpha} = \frac{1}{m} \sum_{i=1}^{m} p_{C_i}(x)
\end{equation}
That is, samples from each cluster are equally likely. 

\paragraph{Evaluation Distribution of $h^*$:} However, for $h^*$ defined above to be effective, one needs \emph{same scale of the nearest neighbors irrespective of the variance of gaussians}. Thus, to apply (and consequently to evaluate) the above hypothesis, we assume that the distribution is
\begin{equation}
    D^{0.5} = \frac{1}{m} \sum_{i=1}^{m} \gamma_i p_{C_i}(x)
\end{equation}
where $\gamma_i$ denotes the proportion of the samples from $C_i$. For the scale to be equivalent, one must obtain a larger sample from clusters with larger variance. 

\paragraph{Correcting this Discrepancy using \rnthm{} Derivative:} Following the reasoning from above, we correct the discrepancy using the \rnthm{} derivative which is given by
\begin{equation}
    f(x) = \sum_{i=1}^{m} \gamma_i I[x \in C_i]
\end{equation}
If we estimate $\gamma_i$ from the data, we get $\gamma_i \approx |C_i|$. 

\paragraph{Final Algorithm for Anomaly Detection:} We further also use a ``filtering'' to ignore small clusters. This leads to the following
\begin{itemize}
    \item[1.] Cluster the sample into clusters $\{C_1, C_2, \cdots, C_m\}$, obtain the $m$ corresponding means $\mu_k$
    \item[2.] Consider only the largest $k < m$ clusters to estimate the inline density $p_I$. This filtering step allows us to estimate $p_I$.
    \item[3.] Obtain the base-scores $h^*(x)$ using \cref{eq:a31} and the cluster centroids.
    \item[4.] Adjust the scores using \rnthm{} derivative correction.
        \begin{equation}
            h^{\dagger}(x) =
            \begin{cases}
                |C_i| d(x, C_i), & \text{if } x \in C_i, C_i \text{ is large} \\
                |C_i| d(x, C_L), & \text{if } x \in C_i, C_L \text{ is the nearest large cluster}
            \end{cases}
            \label{eq:final_cblof}
        \end{equation}
        where $C_L$ is the closest large cluster. 
\end{itemize}
This algorithm is the widely adopted cluster based local outlier factor (CBLOF) \cite{CBLOF}.

We also remark that there exists other modified versions of CBLOF. 
\begin{itemize}
    \item[(A)] ECBLOF (Enhanced Cluster-Based Local Outlier Factor)\cite{ECBLOF}, which assumes 
    \begin{equation}
        D^{0.5} = \frac{1}{m} \sum_{i=1}^{m} p_{C_i}(x)
    \end{equation}
    In this case there is no \rnthm{} derivative correction required.
    \item[(B)] Observe that, we estimated $\gamma_i \approx |C_i|$. However, one can use techniques such as kernel density estimate (KDE) to explicitly estimate the \rnthm{} derivative.
\end{itemize}

\textbf{Key Takeaway:} The most important aspect to note here is that -- \rnthm{} derivative can be applied in wide variety of contexts to obtain practically useful algorithms. Further it also has a flexibility of adapting to various application scenarios depending on the information available.

\section{Empirical Evaluations}
\label{sec:4}
In this section, we aim to show that \textbf{The \rnthm{} derivative correction proposed above demonstrates performance that is on par with, and in many cases surpasses, that of several state-of-the-art techniques.}.

We first discuss the experimental settings in detail in \cref{sec:41}, and discuss the observations in \cref{sec:42}.

\subsection{Experimental Settings}
\label{sec:41}

\begin{figure}
  \centering
    \begin{subfigure}[b]{0.33\linewidth} 
      \includegraphics[width=1.1\linewidth]{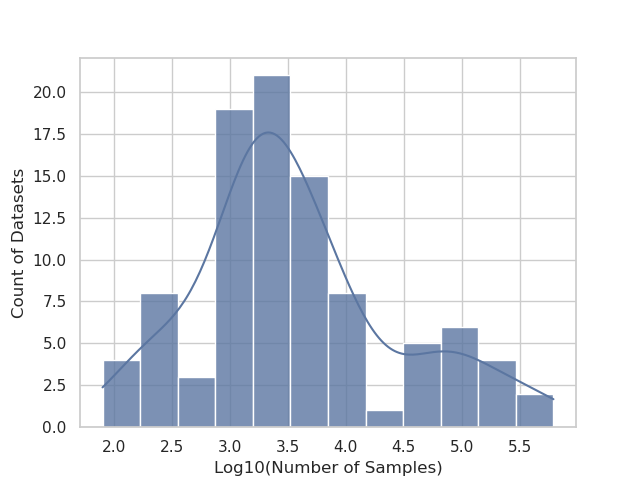}
      \Description{}
      \caption{}
      \label{fig:exp1a}
    \end{subfigure}%
    \hfill %
    \begin{subfigure}[b]{0.33\linewidth} 
      \includegraphics[width=1.1\linewidth]{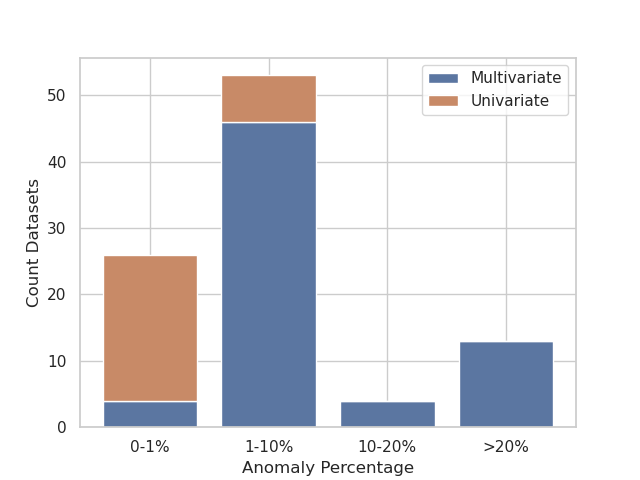}
      \Description{}
      \caption{}
      \label{fig:exp1b}
    \end{subfigure}
    \hfill %
    \begin{subfigure}[b]{0.33\linewidth} 
      \includegraphics[width=1.1\linewidth]{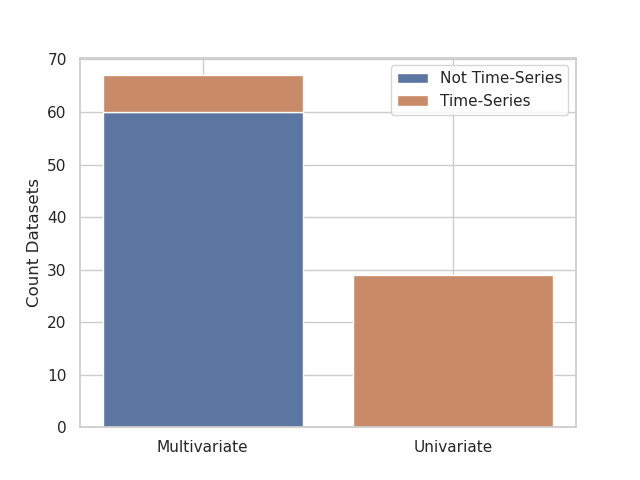}
      \Description{}
      \caption{}
      \label{fig:exp1c}
    \end{subfigure}
  \caption{Overview of Datasets: Observe that the datasets considered have wide range of total sizes, anomaly percentages, and as well as diversity with respect to other characteristics such as Univariate vs Multivariate, Time-Series vs Non-Time-Series. (a) shows the number of datasets with different number of samples. Observe that the sizes vary from $80$ to $619,329$. (b) Shows the number of datasets with respect to percentage anomalies. We consider 4 ranges corresponding to ``very less'', ``less'', ``medium'' and ``large'' number of anomalies. (c) indicates number of datasets which has a time-series characteristic.}
  \label{fig:exp1}
\end{figure}

\subsubsection{Datasets for Empirical Validation:} As discussed earlier, there exists several factors which effect the practical aspects of anomaly detection. Hence, empirical validation must be conducted on a wide variety of datasets. In this article, we consider 96 datasets to cover the range of aspects of anomaly detection. \cref{fig:exp1} shows provides an overview of the datasets considered.
\begin{enumerate}
    \item[(i)] Firstly, we consider both univariate and multivariate datasets. In total we consider 29 univariate datasets and 67 multivariate datasets.
    \item[(ii)] We consider datasets of different sizes ranging from $80$ to $619,329$. \cref{fig:exp1a} shows the distribution of the sizes within the 96 datasets we consider.
    \item[(iii)] We also consider datasets with wide-ranging anomaly percentages, from  0.03\% to 43.51\%. However, for simplicity we consider 4 ranges - $0-1\%$, $1-10\%$, $10-20\%$ and $> 20\%$. There are   22 univariate and 4 multivariate datasets in the range $0-1\%$, 40 multivariate and 7 univariate datasets in the range $1-10\%$, 3 multivariate datasets in the range $10-20\%$ and 16 in $> 20\%$ range. \cref{fig:exp1b} illustrates this. 
    \item[(iv)] We also consider the time-series aspect. However, since non time-series univariate datasets are not complex this aspect is not considered. In total, we consider 29 univariate time-dependent and 7 multivariate time-dependent datasets. \cref{fig:exp1c} illustrates this. 
    \item[(v)] Further, datasets cover multiple domains such as finance, healthcare, e-commerce, industrial systems, telecommunications, astronautics, computer vision, forensics, botany, sociology, linguistics, etc.
\end{enumerate}
We source these datasets from a combination of ESA-ADB dataset \cite{I44} along with six other SWaT datasets \cite{SWaT} and a BATADAL dataset\cite{taormina18battle}

\subsubsection{Evaluation of Anomaly Detection:} We split the datasets into train/test as follows -- We use 70\% of the inline data for training and 30\% of the inline data for testing. However, we only use $15\%$ of the anomalous data for training and 85\% of the anomalous data for testing. This procedure is follows to ensure robustness of the results. Further we made sure that the \emph{anomaly contamination}(15\%) in the train set is less than or equal to the other baseline methods.

Since anomalies are rare and the datasets are imbalanced we use -- \emph{Precision, Recall, AUROC and F1 score} -- for all the evaluations. 

\subsubsection{Baselines and SoTA:} There exists several algorithms for anomaly detection. These algorithms include Local Outlier Factor(LOF), Isolation Forest (IForest), One-class SVM (OCSVM), Autoencoders, Deep Autoencoding Gaussian Mixture Model (DAGMM)\cite{DAGMM}, Quantile LSTM(q-LSTM)\cite{I49}, Deep Quantile Regression \cite{I77}, GNN \cite{GDN}, GAN \cite{I27}, DevNet \cite{DevNet}, MGBTAI\cite{MGBTAI} and d-BTAI\cite{dBTAI} as covered in \cref{tab:algo_list} in \cref{appendix:baseline}. Above is a mix of supervised and unsupervised methods, forming our baselines for comparison on anomalous datasets. Other important algorithms that have been tested along with the above are ECOD \cite{ECOD}, COPOD \cite{COPOD}, KNN, LUNAR \cite{LUNAR}, PCA, DSVDD, NeuTraL-AD \cite{NeuTraL-AD}, ICL \cite{ICL}, SLAD \cite{SLAD}.

\subsubsection{Tuning Hyperparameters:} In the domain of anomaly detection, determining optimal threshold values is crucial due to the inherently rare and imbalanced nature of anomalies. Setting the threshold too low may lead to a high number of false positives, reducing the model's significance, and setting it too high may cause the model to miss critical anomalies. Therefore, aiming at the most optimal model performance, we set the following thresholds in accordance with the suggestions from the literature:
\begin{itemize}
    \item For Autoencoders, the lower threshold is set at the $0.75$th percentile, and the upper threshold is at the $99.25$th percentile of Mean Squared Error(MSE) values.
    \item  All the data points with discriminator scores less than $10$th percentile were considered anomalies for GANs.
    \item DAGMM had a dual threshold setting with high and low thresholds, with two standard deviations above and below the mean
    \item For the tree-based approaches, MGBTAI was set to a minimum clustering threshold of 20\% of the dataset size and leaf level threshold of 4, while for d-BTAI, the minimum clustering threshold was set to 10\% of the dataset size
    \item For deep quantile regression, the lower threshold was at $0.9$th percentile and the upper at $99.1$st percentile of the predicted values
\end{itemize}
As can be observed this process becomes extremely meticulous and takes up most of the experimental time.
For RN-Loss, we automate our threshold calculation process by maximizing the difference in True positive and False positive rates, which are very important metrics obtained from the AUC-ROC curve. This helps us get the best optimal threshold, and SoTA results across datasets and overall algorithms. The thresholds range from 0.001 to 0.999 and can be found for all the respective datasets in \cref{appendix:baseline}, \cref{table:5}.

\subsubsection{Network Architecture:} In the \underline{supervised} case -- We use two architectures in our study. First, \textbf{RN-Net} is a ReLU feedforward network with RN-Loss, comprising 64 hidden units in a binary classification setting. We train for 50 epochs using the Adam optimizer. Additionally, we use batch normalization~\cite{Norm}, dropout~\cite{Dropout}, and early stopping with a threshold of 10 epochs. We reduce our learning rate by half every 5 epochs until it reaches $10^{-6}$. Parallel to this, we integrated L2 regularisation with RN-Net and noticed a further step-up in performance across datasets. Similarly, to demonstrate the flexibility and adaptability of RN-Loss, we create \textbf{RN-LSTM}: A LSTM with 32 hidden units coupled with the RN-Loss function.

In the \underline{unsupervised} case, we use 3 different algorithms which are modified using \rnthm{} derivative -- (i) \textbf{Kmeans(CBLOF)} uses K-Means to estimate the clusters and applies CBLOF (ii) \textbf{Kmeans(CBLOF, Mod.)} estimates the \rnthm{} using kernel density estimates, and (iii) \textbf{dBTAI(Mod.)} modies the dBTAI algorithm in \cite{dBTAI} using the ECBLOF loss. 

\textbf{Remark:} The code and other resources are provided in the the anonymous repo - \url{https://anonymous.4open.science/r/RN_Derivative_Official/}. Note that we perform our experiments on 96 different datasets and 23 different algorithms. To maintain clarity and focus, we report only the most salient aspects of the results. Please refer to the repo for the extensive list of all the results.

\subsection{Empirical Findings}
\label{sec:42}

\begin{figure}
  \centering
    \begin{subfigure}[b]{0.3\linewidth} 
      \includegraphics[width=1\linewidth]{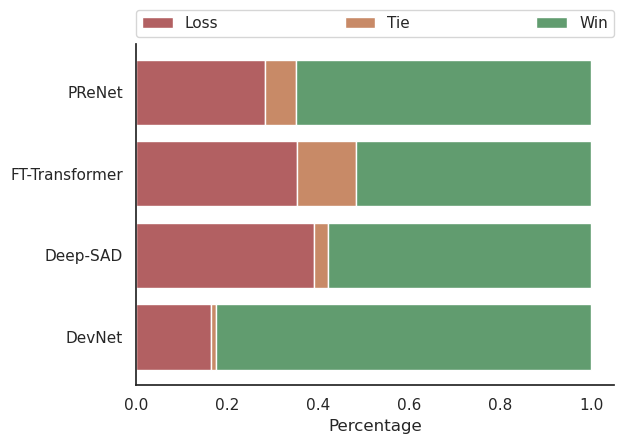}
      \caption{AUROC}
      \label{fig:exp2a}
    \end{subfigure}%
    \hfill %
    \begin{subfigure}[b]{0.3\linewidth} 
      \includegraphics[width=1\linewidth]{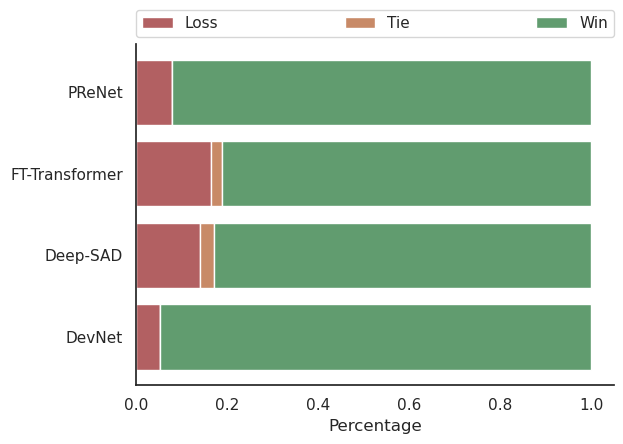}
      \caption{F1}
      \label{fig:exp2b}
    \end{subfigure}%
    \hfill %
    \begin{subfigure}[b]{0.3\linewidth}
    \scriptsize
      \begin{tabular}{c|c}      
      \hline
       Algorithm & Avg. Rank \\
       \hline
        RN-Net      & 2.088542 \\
        Deep-SAD    & 2.585938 \\
        FT-Transformer& 2.617647 \\
        PReNet      & 3.107955 \\
        DevNet      & 3.681319 \\
        \hline
      \end{tabular}
      \vspace{3em}
    \caption{Avg. Rank}
      \label{fig:exp2f}
    \end{subfigure}%
    
    \begin{subfigure}[b]{0.33\linewidth} 
      \includegraphics[width=1\linewidth]{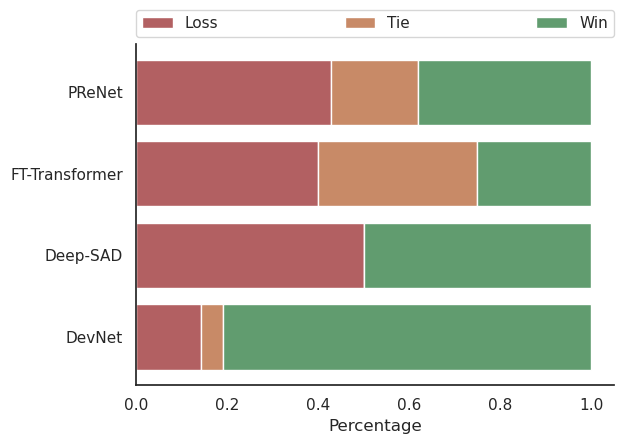}
      \caption{0-1\% Anomalies (AUROC)}
      \label{fig:exp2c}
    \end{subfigure}%
    \hfill %
    \begin{subfigure}[b]{0.33\linewidth} 
      \includegraphics[width=1\linewidth]{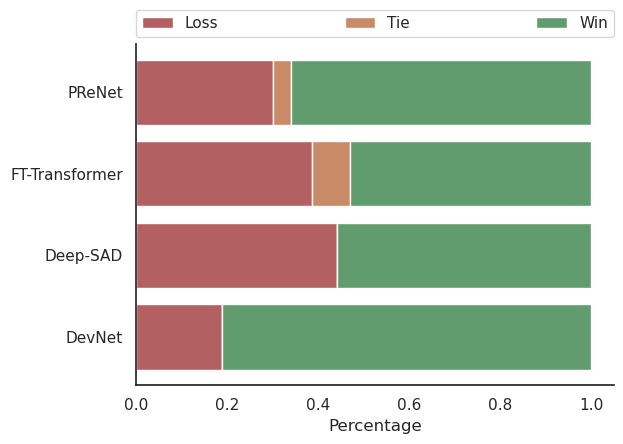}
      \caption{1-10\% Anomalies (AUROC)}
      \label{fig:exp2d}
    \end{subfigure}%
    \hfill %
    \begin{subfigure}[b]{0.33\linewidth} 
      \includegraphics[width=1\linewidth]{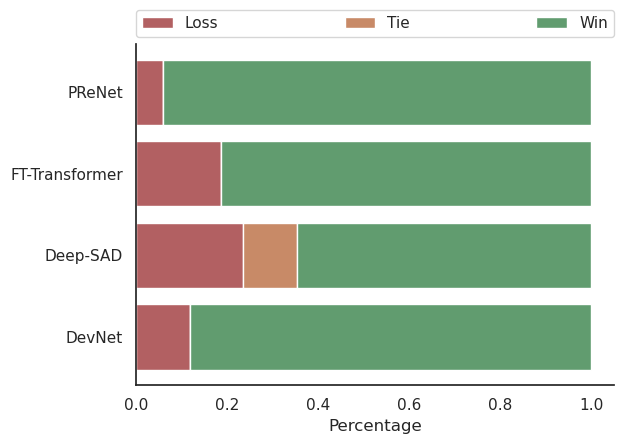}
      \caption{$>10\%$ Anomalies (AUROC)}
      \label{fig:exp2e}
    \end{subfigure}%

    \begin{subfigure}[b]{0.33\linewidth} 
      \includegraphics[width=1\linewidth]{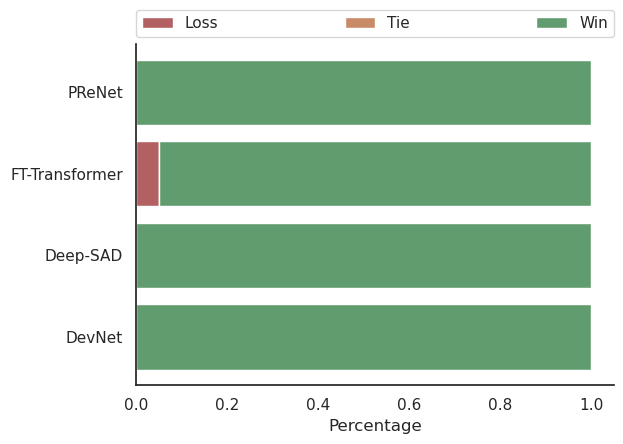}
      \caption{0-1\% Anomalies (F1)}
      \label{fig:exp2g}
    \end{subfigure}%
    \hfill %
    \begin{subfigure}[b]{0.33\linewidth} 
      \includegraphics[width=1\linewidth]{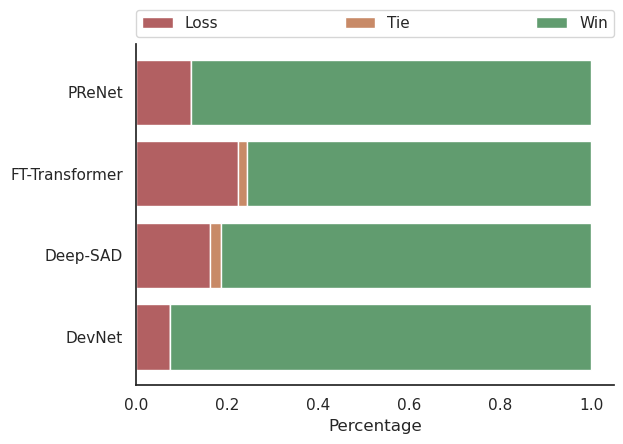}
      \caption{1-10\% Anomalies (F1)}
      \label{fig:exp2h}
    \end{subfigure}%
    \hfill %
    \begin{subfigure}[b]{0.33\linewidth} 
      \includegraphics[width=1\linewidth]{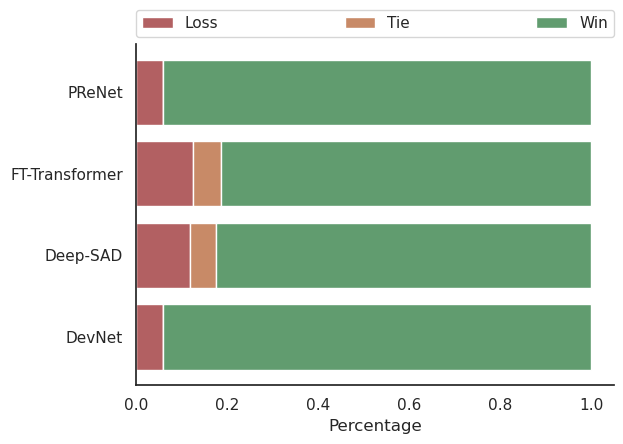}
      \caption{$>10\%$ Anomalies (F1)}
      \label{fig:exp2i}
    \end{subfigure}%
    \Description{}
  \caption{Performance of RN-Loss for supervised anomaly detection. \legendsquare{color1} indicates the percentage of the datasets in which the specific algorithm performs better than RN-Net. \legendsquare{color2} indicates the percentage of the datasets in which the specific algorithm tied with RN-Net. \legendsquare{color3} indicates the percentage of the datasets in which the specific algorithm performs worse than RN-Net. (a) shows the performance of RN-Loss with respect to AUROC. (b) shows the performance of the RN-Loss with respect to F1-score. (d)-(f) shows the performance of RN-Loss when considering various ranges of anamoly percentages.}
  \label{fig:exp2}
\end{figure}

\begin{figure}
    \centering
    \begin{subfigure}[b]{0.45\linewidth} 
      \includegraphics[width=\linewidth]{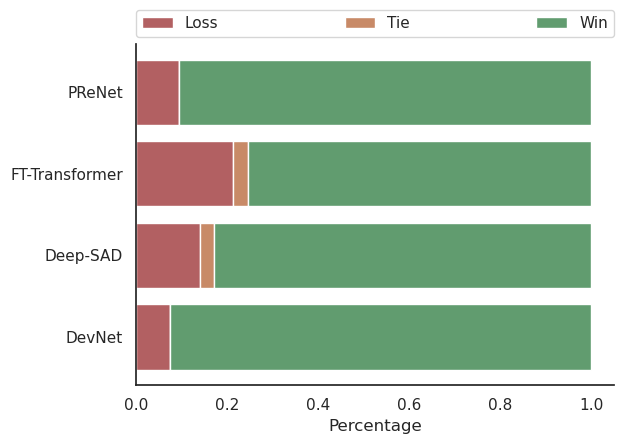}
      \caption{Multivariate (F1)}
      \label{fig:exp3a}
    \end{subfigure}%
    \hfill %
    \begin{subfigure}[b]{0.45\linewidth} 
      \includegraphics[width=\linewidth]{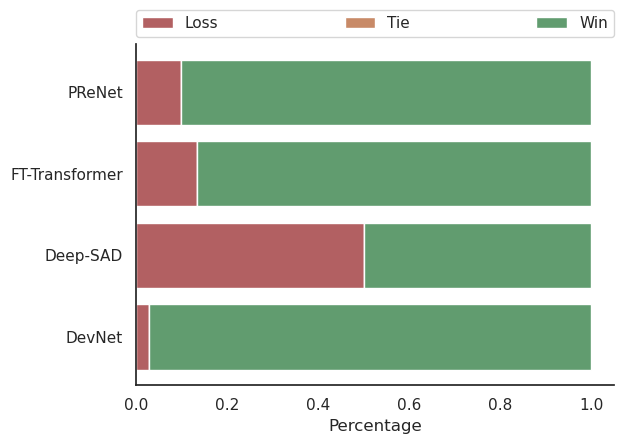}
      \caption{Time-Series (F1)}
      \label{fig:exp3b}
    \end{subfigure}%
    \Description{}
    \caption{Performance of RN-Net on Multivariate and Time-Series Datasets. (a) illustrates the results on multivariate datasets. (b) shows the results on time-series datasets.}
  \label{fig:exp3}
\end{figure}

\subsubsection{Performance of RN-Net under supervised settings}

\cref{fig:exp2} summarizes the results of RN-Loss when compared with recent state-of-the-art approaches -- DevNet\cite{DevNet}, FT-Transformer\cite{FTT}, Deep-SAD\cite{DeepSAD}, PReNet\cite{PReNet}. The key findings are :
\begin{enumerate}
    \item[(1)] Across datasets, RN-Net performs better than the recent SoTA approaches. 
    \item[(2)] Performance with respect to F1 score of RN-Loss is better than AUROC. This is thanks to the optimal and simple threshold finding for RN-Loss described above. 
    \item[(3)] Across all the datasets, the average rank (lower is better) for RN-Net is 2.08, Deep-SAD is 2.58, FT-Transformer is 2.61, PReNet is 3.10, DevNet is 3.68. These ranks are with respect to AUROC.
    \item[(4)] Performance of RN-Net is also superior across different anomaly ratios. In the region of $0-1\%$ anomalies, with respect to AUROC, RN-Loss is comparable to the SoTA approaches, with respect to F1 scores RN-Net outperforms all competing approaches. For $0-10\%$ anomalies and for $>10\%$ anomalies, RN-Net outperforms all competing approaches with respect to both AUROC and F1 scores.
\end{enumerate}

\subsubsection{Performance of RN-Net -- Multivariate Datasets:} RN-Net performs the best on 44 out of 67 datasets. The rest of the algorithms demonstrate varying performance across multiple datasets. FT-Transformer achieves the best results on 8 datasets in terms of precision, 2 datasets for recall, 9 datasets for F1-score, and 7 datasets for AUCROC. DeepSAD shows superior performance on 3 datasets for precision, 10 for recall, 9 for F1-score, and 8 for AUC-ROC. Meanwhile, PReNet achieves top performance on 2 datasets in terms of precision, 6 for recall, none for F1score, and 7 for AUCROC. This is summarized in \cref{fig:exp3a}.

\subsubsection{Performance of RN-Net -- Time-Series Datasets (Univariate and Multivariate):} Across 29 univariate time-series benchmark datasets RN-Net consistently demonstrates superior performance. Specifically, it achieves the highest precision and recall on 20 datasets, the best F1 score on 21, and the top AUC-ROC on 17. Recent models, including FT-Transformer, DeepSAD, and PReNet, yield only sporadic wins. RN-LSTM -- designed primarily to illustrate the adaptability of the proposed RN-loss—also performs competitively, though its effectiveness is somewhat limited by fixed timestep constraints and the extreme sparsity of anomalies. In multivariate settings (evaluated on SWaT, BATADAL, and ESA-ADB), RN-Net again exhibits optimal performance, achieving the highest F1 scores across \emph{all} datasets and maintaining consistently strong recall and AUC-ROC. These results on all the time-series datasets are summarized in \cref{fig:exp3b}. 

\subsubsection{Performance of \rnthm{} derivative correction in unsupervised setting:} Recall that the \rnthm{} correction is leads to 3 algorithms KMeans-CBLOF, KMeans-CBLOF-Mod, dbTAI-Mod (See \cref{sec:32}). As baselines, we consider both \emph{unsupervised methods} - GAN, q-LSTM (sigmoid), DAGMM as well as the recent state-of-the-art \emph{self-supervised methods} - SLAD, NeuTraL-AD, ICL, GAN. 

Across 96 diverse benchmarks spanning univariate and multivariate, time-series and non-time-series domains, the KMeans-CBLOF variant consistently outperformed the other \rnthm{}-corrected methods, achieving the highest AUROC on 37 datasets. Collectively, the trio -- KMeans-CBLOF, KMeans-CBLOF-Mod, and dbTAI -- attained top AUROC performance in 56 out of 96 datasets ($\sim$58\%), demonstrating broad competitiveness. Notably, KMeans-CBLOF was particularly effective under sparse anomaly conditions (0–1\% rate), leading in 16 of 26 such cases. While dbTAI achieved fewer AUROC wins, it attained the highest mean F1 score (0.197) and recall (0.703), suggesting a recall-focused trade-off and also displayed improved stability relative to KMeans-CBLOF-Mod. Among non-\rnthm{} baselines, q-LSTM (sigmoid) yielded the highest average AUROC (0.833), but was outperformed by the best of the \rnthm{}-corrected methods in 39 datasets by a margin of at least 0.05 AUROC. The largest observed AUROC gap, approximately 0.43, occurred on the \texttt{vowels} dataset. \cref{tab:1} reports the average rank of all algorithms, with \rnthm{}-corrected variants occupying 3 of the top 4 positions. Pairwise AUROC comparisons in \cref{fig:exp4} further validate the efficacy of simple derivative correction in enhancing classical unsupervised methods.

\begin{table}[t]
    \centering
    \caption{Average Rank (w.r.t AUROC) of Unsupervised Methods for Anomaly Detection. Observe that the \rnthm{} correction methods take 3 out of the top 4 positions.}
    \label{tab:1}
    \begin{tabular}{|l|c|}
    \hline
    Algorithm   &   Avg. Rank \\
    \hline
    KMeans-CBLOF    & 2.656627 \\
    KMeans-CBLOF-Mod    & 2.861446 \\
    q-LSTM (sigmoid)    & 3.086207 \\
    dbTAI-Mod   & 3.446429 \\
    DAGMM   & 3.859375 \\
    GAN & 4.861702 \\
    SLAD    & 5.231481 \\
    NeuTraL-AD  & 5.456897 \\
    ICL & 5.527778 \\
    \hline
    \end{tabular}
\end{table}

\begin{figure}
    \centering
    \includegraphics[width=\linewidth]{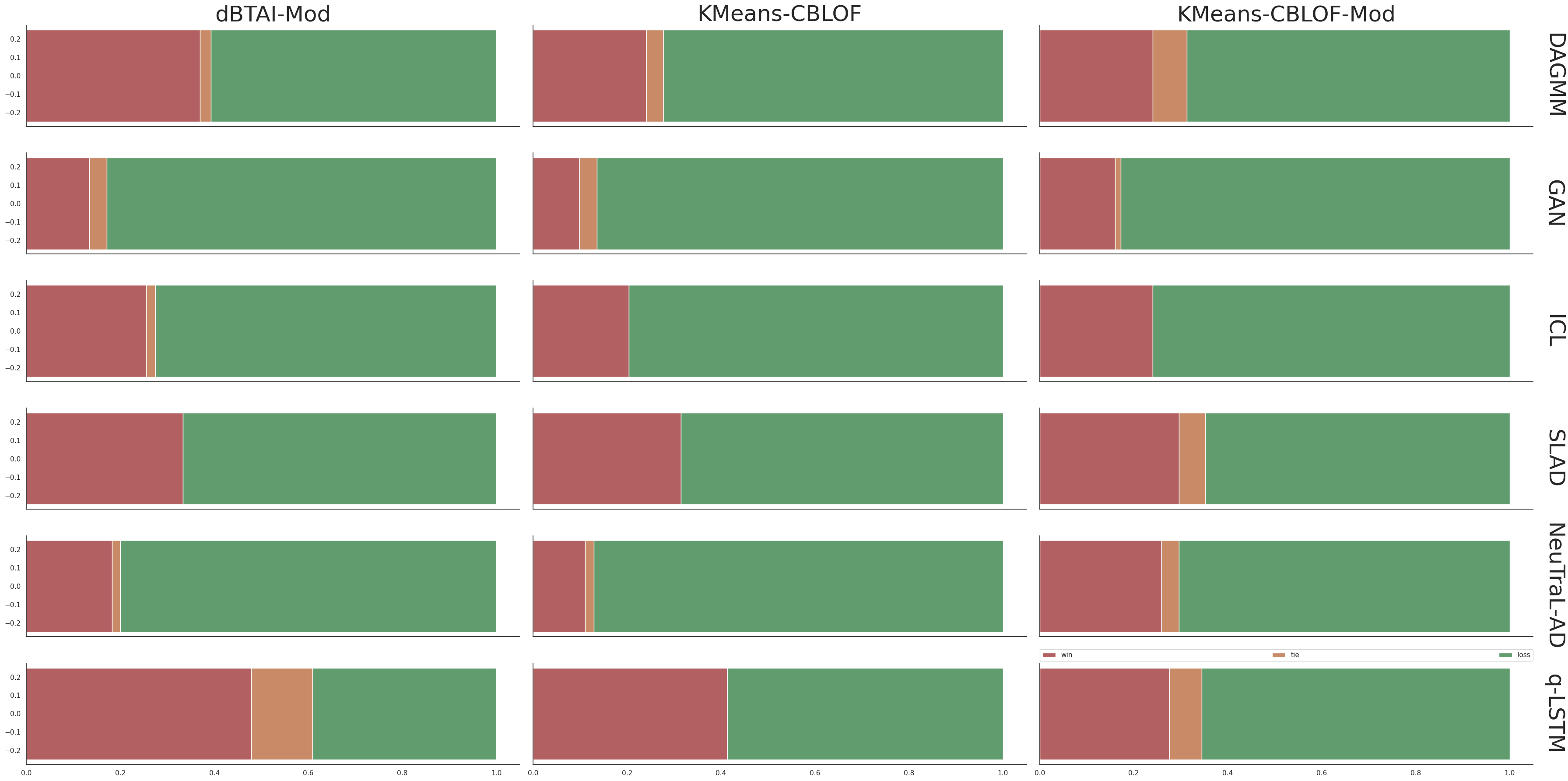}
    \Description{}
    \caption{Performance of \rnthm{} derivative correction of unsupervised algorithms with respect to AUROC. Observe that the \rnthm{} derivative corrected unsupervised algorithms -- dbTAI(Mod.), KMeans-CBLOF and KMeans-CBLOF(Mod.) perform better than recent state-of-the-art algorithms such as ICL, NeuTraL, SLAD etc.  \legendsquare{color1} indicates the percentage of datasets where the row algorithm outperforms the column algorithm. \legendsquare{color2} indicates the percentage where both algorithms perform equally. \legendsquare{color3} indicates the percentage where the row algorithm underperforms relative to the column algorithm. }
  \label{fig:exp4}
\end{figure}

\section{Conclusion}

Anomaly detection is a fundamental problem across multiple domains. Formally, an anomaly is any sample that does not belong to the underlying data distribution. However, identifying anomalies is challenging, particularly when the data distribution exhibits high variability. Despite its importance, the theoretical foundations of anomaly detection remain underexplored.

\emph{What is the right principle to design loss function for anomaly detection?} We show that the right principle should correct the  discrepancies between the distributions. This is easily achieved by weighing the generic loss function with \rnthm{} derivative. We prove this by establishing the PAC learnability of anomaly detection. We refer to this approach as RN-Loss. Notably, we show that (supervised) weight-adjusted loss functions and unsupervised Cluster-Based Local Outlier Factor (CBLOF) naturally emerge as performant and conceptual instances of this correction mechanism.

Empirical evaluations across 96 datasets demonstrate that weighting a standard loss function by the \rnthm{} derivative enhances performance, making RN-Loss a robust, efficient, and adaptable solution that outperforms state-of-the-art methods under varying anomaly contamination levels.

The integral representation of the weighted loss function via \rnthm{} derivative can be used for imbalanced class boundaries. An interesting regularization interpretation can also be handy in future. The sparsity of the derivative implies a sparsity-inducing regularization; conversely, if it is dense, it might be interpreted as a smoothing regularizer.




\bibliography{main}

\newpage

\appendix

\section{Literature Review}
\label{sec:a5}
Throughout the years, much research has been conducted in anomaly detection with a multitude of explored methods such as in \cite{I31, I51, I37, I38, I22, I39, I40, I41}. Since then, this interest has increased significantly as various domains such as cybersecurity \cite{I17, I18}, fraud detection, and healthcare \cite{I20, I21} became more relevant. The work on learning from imbalanced datasets was proposed in AAAI 2000 workshop and highlighted research on two major problems, types of imbalance that hinder the performance of standard classifiers and the suitable approaches for the same. \cite{I68} also showed that the class imbalance problem affects not only standard classifiers like decision trees but also Neural Networks and SVMs. The work by \cite{I66,I67} gave a further boost to research in imbalanced data classification. Extensive research was done on unsupervised learning methods to address issues such as relying on static thresholds, in turn struggling to adapt to dynamic data, resulting in high false positives and missed anomalies. However, the work by \cite{I70} observed that unsupervised anomaly detection can be computationally intensive, especially in high-dimensional datasets. Jacob and Tatbul \cite{I71} delved into explainable anomaly detection in time series using real-world data, yet deep learning-based time-series anomaly detection models were not thoroughly explored well enough. With significant growth in applying various ML algorithms to detect anomalies, there has been an avalanche of anomaly benchmarking data \cite{I43}, \cite{I18}, \cite{I72}, as well as empirical studies of the performances of existing algorithms \cite{I73}, \cite{I74} on different benchmark data. Due to the importance of the problem, there have also been efforts to produce benchmarks such as AD-Bench~\cite{I43} and ESA-ADB~\cite{I44}. Researchers have critically examined the suitability of evaluation metrics for machine learning methods in anomaly detection. Kim et al. \cite{I76} exposed the limitations of the F1-score with point adjustment, both theoretically and experimentally.

To conclude, recent benchmarking studies, concentrated on deep-learning-based anomaly detection techniques mostly, have not examined the performance across varying types of anomalies, such as singleton/point, small, and significantly large numbers, nor across different data types, including univariate, multivariate, temporal, and non-temporal data. Additionally, there has been a lack of exploration into how anomalies should be identified when their frequency is high. These observations prompt several critical questions. Is the current SoTA algorithm the most effective? Are we reaching the peak in anomaly detection using Deep Learning approaches? Are unsupervised learning algorithms truly better than supervised learning algorithms?

\section{Baselines and Additional Tables}
\label{appendix:baseline}

\begin{table*}[h!]
  \centering
  \scriptsize
  \setlength{\tabcolsep}{1mm}
  \renewcommand{\arraystretch}{1.5}
  \caption{Descriptions and hyperparameter settings of SOTA algorithms benchmarked in this study}
  \label{tab:algo_list}
  \begin{tabular}{|p{0.12\linewidth}|p{0.85\linewidth}|}
    \hline
    \textbf{Algorithm} & \textbf{Anomaly Detection Approach} \\
    \hline
    LOF & Uses data point densities to identify an anomaly by measuring how isolated a point is relative to its nearest neighbors in the feature space. Implemented using the Python Outlier Detection (PyOD) library with default parameters. Trained on 70\% of data and tested on the entire dataset. \\
    \hline
    Iforest & Ensemble-based algorithm that isolates anomalies by constructing decision trees. Reported to perform well in high-dimensional data. The algorithm efficiently separates outliers by requiring fewer splits in the decision tree compared to normal data points. Implemented using scikit-learn with default parameters. Trained on 70\% of data and tested on the entire dataset. \\
    \hline
    OCSVM & Constructs a hyperplane in a high-dimensional space to separate normal data from anomalies. Implemented using scikit-learn with default parameters. Trained on 70\% normal data, or whatever normal data was available. \\
    \hline
    AutoEncoder & Anomalies are identified based on the reconstruction errors generated during the encoding-decoding process. \textit{Requires training on normal data.} Implemented using Keras. Lower threshold set at the 0.75th percentile, and upper threshold at the 99.25th percentile of the Mean Squared Error (MSE) values. Trained on 70\% normal data, or whatever normal data was available. Anomalies are detected by comparing the reconstruction error with the predefined thresholds. \\
    \hline
    DAGMM & Combines .{autoencoder} and Gaussian mixture models to model the data distribution and identify anomalies. It has a compression network to process low-dimensional representations, and the Gaussian mixture model helps capture data complexity. \textit{This algorithm requires at least 2 anomalies to be effective.} It is trained on 70\% normal data, or whatever normal data was available. Anomalies are detected by calculating anomaly scores, with thresholds set at two standard deviations above and below the mean anomaly score. \\
    \hline
    LSTM & Trained on a normal time-series data sequence. Acts as a predictor, and the prediction error, drawn from a multivariate Gaussian distribution, detects the likelihood of anomalous behavior. Implemented using Keras. Lower threshold set at the 5th percentile, and upper threshold at the 95th percentile of the Mean Squared Error (MSE) values. Trained on 70\% normal data, or whatever normal data was available. Anomalies are detected when the prediction error lies outside the defined thresholds. \\
    \hline
    qLSTM & Augments LSTM with quantile thresholds to define the range of normal behavior within the data. Implementation follows the methodology described in the authors' paper, which applies quantile thresholds to LSTM predictions. Anomalies are detected when the prediction error falls outside the defined quantile range. \\
    \hline
    QREG & A multilayered LSTM-based RNN forecasts quantiles of the target distribution to detect anomalies. The core mathematical principle involves modeling the target variable's distribution using multiple quantile functions. Lower threshold set at the 0.9th percentile, and upper threshold at the 99.1st percentile of the predicted values. Anomalies are detected when the predicted value lies outside these quantile thresholds. Trained on 70\% data and tested on the entire dataset. \\
    \hline
    Elliptic Envelope & Fits an ellipse around the central multivariate data points, isolating outliers. It needs a contamination parameter of 0.1 by default, with a support fraction of 0.75, and uses Mahalanobis distance for multivariate outlier detection. Implemented using default parameters from the sklearn package. Trained on 70\% of data and tested on the entire dataset. Anomalies are detected when data points fall outside the fitted ellipse. \\
    \hline
    DevNet & A Deep Learning-based model designed specifically for anomaly detection tasks. Implemented using the Deep Learning-based Outlier Detection (DeepOD) library. Anomalies are detected based on the deviation score, with a threshold defined according to the model's performance and expected anomaly rate. Trained on 70\% data and tested on the entire dataset; \textit{it requires atleast 2 anomalies in its training set to function, and for optimal performance, it is recommended to include at least 2\% anomalies in the training data.} \\
    \hline
    GAN & Creates data distributions and detects anomalies by identifying data points that deviate from the generated distribution. It consists of generator and discriminator networks trained adversarially. Implemented using Keras. All data points whose discriminator score lies in the lowest 10th percentile are considered anomalies. Trained on 70\% normal data. \\
    \hline
    GNN & GDN, which is based on graph neural networks, learns a graph of relationships between parameters and detects deviations from the patterns. Implementation follows the methodology described in the authors' paper. \\
    \hline
    MGBTAI & An unsupervised approach that leverages a multi-generational binary tree structure to identify anomalies in data. Minimum clustering threshold set to 20\% of the dataset size and leaf level threshold set to 4. Used k-means clustering function. No training data required.\\
    \hline
    dBTAI & Like MGBTAI, it does not rely on training data. It adapts dynamically as data environments change.The small cluster threshold is set to 2\% of the data size. The leaf level threshold is set to 3. The minimum cluster threshold is set to 10\% of the data size and the number of clusters are 2 (for KMeans clustering at each split). The split threshold is 0.9 (used in the binary tree function). The anomaly threshold is determined dynamically using the knee/elbow method on the cumulative sum of sorted anomaly scores. The kernel density uses a gaussian kernel with default bandwidth and uses imbalance ratio to weight the density ratios. Used k-means clustering function. No training data required.\\
    \hline
    FTTransformer & It is a sample adaptation of the original transformer architecture for tabular data. The model transforms all features (categorical and numerical) to embeddings and applies a stack of Transformer layers to the embeddings. However, as stated in the original paper's \cite{FTT} limitations: FTTransformer requires more resources (both hardware and time) for training than simple models such as ResNet and may not be easily scaled to datasets when the number of features is “too large”. \\
    \hline
    DeepSAD & It is a generalization of the unsupervised Deep SVDD method to the semi-supervised anomaly detection setting and thus needs labeled data for training. It is also considered as an information-theoretic framework for deep anomaly detection.\\
    \hline
    PReNet & It has a basic ResNet with input and output convolution layers, several residual blocks (ResBlocks) and a recurrent layer implemented using a LSTM. It is particularly created for the task of image deraining as mentioned in \cite{PReNet}.\\
    \hline
  \end{tabular}
\end{table*}

\begin{table*}[t]
\centering
\begin{minipage}[t]{0.5\textwidth}
    \centering
    \setlength{\tabcolsep}{2mm}
    \tiny
    \begin{tabular}{|l|l|l|l|l|l|}
    \hline
    \textbf{Dataset} & \textbf{Size} & \textbf{Dimension} & \textbf{\# Anomalies} & \textbf{\% Anomalies} & \textbf{Domain} \\ \hline
    ALOI & 49534 & 27 & 1508 & 3.04 & Image \\ \hline
    annthyroid & 7200 & 6 & 534 & 7.42 & Healthcare \\ \hline
    backdoor & 95329 & 196 & 2329 & 2.44 & Network \\ \hline
    breastw & 683 & 9 & 239 & 34.99 & Healthcare \\ \hline
    campaign & 41188 & 62 & 4640 & 11.27 & Finance \\ \hline
    cardio & 1831 & 21 & 176 & 9.61 & Healthcare \\ \hline
    Cardiotocography & 2114 & 21 & 466 & 22.04 & Healthcare \\ \hline
    celeba & 202599 & 39 & 4547 & 2.24 & Image \\ \hline
    cover & 286048 & 10 & 2747 & 0.96 & Botany \\ \hline
    donors & 619326 & 10 & 36710 & 5.93 & Sociology \\ \hline
    fault & 1941 & 27 & 673 & 34.67 & Physical \\ \hline
    fraud & 284807 & 29 & 492 & 0.17 & Finance \\ \hline
    glass & 214 & 7 & 9 & 4.21 & Forensic \\ \hline
    Hepatitis & 80 & 19 & 13 & 16.25 & Healthcare \\ \hline
    http & 567498 & 3 & 2211 & 0.39 & Web \\ \hline
    InternetAds & 1966 & 1555 & 368 & 18.72 & Image \\ \hline
    Ionosphere & 351 & 33 & 126 & 35.9 & Mineralogy \\ \hline
    landsat & 6435 & 36 & 1333 & 20.71 & Astronautics \\ \hline
    letter & 1600 & 32 & 100 & 6.25 & Image \\ \hline
    Lymphography & 148 & 18 & 6 & 4.05 & Healthcare \\ \hline
    magic.gamma & 19020 & 6 & 260 & 1.37 & Physical \\ \hline
    mammography & 11183 & 6 & 260 & 2.32 & Healthcare \\ \hline
    mnist & 7603 & 100 & 700 & 9.21 & Image \\ \hline
    musk & 3062 & 166 & 97 & 3.17 & Chemistry \\ \hline
    optdigits & 5216 & 64 & 150 & 2.88 & Image \\ \hline
    PageBlocks & 5393 & 10 & 510 & 9.46 & Document \\ \hline
    pendigits & 6870 & 16 & 156 & 2.27 & Image \\ \hline
    Pima & 768 & 8 & 268 & 34.9 & Healthcare \\ \hline
    satellite & 6435 & 36 & 2036 & 31.64 & Astronautics \\ \hline
    satimage-2 & 5803 & 36 & 71 & 1.22 & Astronautics \\ \hline
    shuttle & 49097 & 9 & 3511 & 7.15 & Astronautics \\ \hline
    skin & 245057 & 3 & 50859 & 20.75 & Image \\ \hline
    smtp & 95156 & 3 & 30 & 0.03 & Web \\ \hline
    SpamBase & 4207 & 57 & 1679 & 39.91 & Document \\ \hline
    speech & 3686 & 400 & 61 & 1.65 & Linguistics \\ \hline
    Stamps & 340 & 9 & 31 & 9.12 & Document \\ \hline
    thyroid & 3772 & 6 & 93 & 2.47 & Healthcare \\ \hline
    vertebral & 240 & 6 & 30 & 12.5 & Biology \\ \hline
    vowels & 1456 & 12 & 50 & 3.43 & Linguistics \\ \hline
    Waveform & 3443 & 21 & 100 & 2.9 & Physics \\ \hline
    WBC & 223 & 9 & 10 & 4.48 & Healthcare \\ \hline
    WDBC & 367 & 30 & 10 & 2.72 & Healthcare \\ \hline
    Wilt & 4819 & 5 & 257 & 5.33 & Botany \\ \hline
    wine & 129 & 13 & 10 & 7.75 & Chemistry \\ \hline
    WPBC & 198 & 33 & 47 & 23.74 & Healthcare \\ \hline
    yeast & 1484 & 8 & 507 & 34.16 & Biology \\ \hline
    CIFAR10 & 5263 & 512 & 263 & 5 & Image \\ \hline
    FashionMNIST & 6315 & 512 & 315 & 5 & Image \\ \hline
    MNIST-C & 10000 & 512 & 500 & 5 & Image \\ \hline
    MVTec-AD & 292 & 512 & 63 & 21.5 & Image \\ \hline
    SVHN & 5208 & 512 & 260 & 5 & Image \\ \hline
    Agnews & 10000 & 768 & 500 & 5 & NLP \\ \hline
    Amazon & 10000 & 768 & 500 & 5 & NLP \\ \hline
    Imdb & 10000 & 768 & 500 & 5 & NLP \\ \hline
    Yelp & 10000 & 768 & 500 & 5 & NLP \\ \hline
    20newsgroups & 3090 & 768 & 155 & 5 & NLP \\ \hline
    BATADAL 04 & 4177 & 43 & 219 & 5.24 & Industrial \\ \hline
    SWaT 1 & 50400 & 51 & 4466 & 8.86 & Industrial \\ \hline
    SWaT 2 & 86400 & 51 & 4216 & 4.88 & Industrial \\ \hline
    SWaT 3 & 86400 & 51 & 3075 & 3.56 & Industrial \\ \hline
    SWaT 4 & 86319 & 51 & 37559 & 43.51 & Industrial \\ \hline
    SWaT 5 & 86400 & 51 & 2167 & 2.51 & Industrial \\ \hline
    SWaT 6 & 54000 & 51 & 3138 & 5.81 & Industrial \\ \hline
    ecoli & 336 & 7 & 9 & 2.68 & Healthcare \\ \hline
    cmc & 1473 & 9 & 17 & 1.15 & Healthcare \\ \hline
    lympho h & 148 & 18 & 6 & 4.05 & Healthcare \\ \hline
    wbc h & 378 & 30 & 21 & 5.56 & Healthcare \\ \hline
    \end{tabular}
    \caption{Multivariate Datasets Characterisation}
    \label{table:2}
\end{minipage}
\end{table*}

\begin{table*}[t]
\begin{minipage}[t]{0.5\textwidth}
    \centering
    \setlength{\tabcolsep}{2mm}
    \tiny
    \begin{tabular}{|l|c|l|c|}
    \hline
    \textbf{Dataset}     & \textbf{Optimal Threshold} & \textbf{Dataset}    & \textbf{Optimal Threshold} \\ \hline
    ALOI                & 0.007       & yahoo1          & 0.130 \\ \hline
    annthyroid          & 0.318       & yahoo2          & 0.594 \\ \hline
    backdoor            & 0.490       & yahoo3          & 0.415 \\ \hline
    breastw             & 0.751       & yahoo5          & 0.233 \\ \hline
    campaign            & 0.009       & yahoo6          & 0.127 \\ \hline
    cardio              & 0.445       & yahoo7          & 0.278 \\ \hline
    Cardiotocography    & 0.438       & yahoo8          & 0.243 \\ \hline
    celeba              & 0.006       & yahoo9          & 0.251 \\ \hline
    cover               & 0.0002      & Speed\_6005     & 0.483 \\ \hline
    donors              & 0.221       & Speed\_7578     & 0.463 \\ \hline
    fault               & 0.410       & Speed\_t4013    & 0.319 \\ \hline
    fraud               & 0.197       & TravelTime\_387 & 0.266 \\ \hline
    glass               & 0.502       & TravelTime\_451 & 0.151 \\ \hline
    Hepatitis           & 0.463       & Occupancy\_6005 & 0.217 \\ \hline
    http                & 0.928       & Occupancy\_t4013& 0.363 \\ \hline
    InternetAds         & 0.684       & yahoo\_syn1     & 0.464 \\ \hline
    Ionosphere          & 0.393       & yahoo\_syn2     & 0.396 \\ \hline
    landsat             & 0.253       & yahoo\_syn3     & 0.534 \\ \hline
    letter              & 0.695       & yahoo\_syn5     & 0.449 \\ \hline
    Lymphography        & 0.474       & yahoo\_syn6     & 0.421 \\ \hline
    magic.gamma         & 0.074       & yahoo\_syn7     & 0.437 \\ \hline
    mammography         & 0.220       & yahoo\_syn8     & 0.335 \\ \hline
    mnist               & 0.319       & yahoo\_syn9     & 0.457 \\ \hline
    musk                & 0.999       & aws1            & 0.125 \\ \hline
    optdigits           & 0.022       & aws2            & 0.444 \\ \hline
    PageBlocks          & 0.398       & aws3            & 0.238 \\ \hline
    pendigits           & 0.274       & aws\_syn1       & 0.387 \\ \hline
    Pima                & 0.455       & aws\_syn2       & 0.639 \\ \hline
    satellite           & 0.114       & aws\_syn3       & 0.398 \\ \hline
    satimage-2          & 0.222       &                 &       \\ \hline
    shuttle             & 0.861       &                 &       \\ \hline
    skin                & 0.005       &                 &       \\ \hline
    smtp                & 0.001       &                 &       \\ \hline
    SpamBase            & 0.409       &                 &       \\ \hline
    speech              & 0.361       &                 &       \\ \hline
    Stamps              & 0.487       &                 &       \\ \hline
    thyroid             & 0.341       &                 &       \\ \hline
    vertebral           & 0.446       &                 &       \\ \hline
    vowels              & 0.414       &                 &       \\ \hline
    Waveform            & 0.311       &                 &       \\ \hline
    WBC                 & 0.547       &                 &       \\ \hline
    WDBC                & 0.837       &                 &       \\ \hline
    Wilt                & 0.321       &                 &       \\ \hline
    wine                & 0.436       &                 &       \\ \hline
    WPBC                & 0.417       &                 &       \\ \hline
    yeast               & 0.406       &                 &       \\ \hline
    CIFAR10             & 0.325       &                 &       \\ \hline
    FashionMNIST        & 0.409       &                 &       \\ \hline
    MNIST-C             & 0.004       &                 &       \\ \hline
    MVTec-AD            & 0.549       &                 &       \\ \hline
    SVHN                & 0.504       &                 &       \\ \hline
    Agnews              & 0.041       &                 &       \\ \hline
    Amazon              & 0.103       &                 &       \\ \hline
    Imdb                & 0.061       &                 &       \\ \hline
    Yelp                & 0.170       &                 &       \\ \hline
    20newsgroups        & 0.139       &                 &       \\ \hline
    BATADAL\_04         & 0.413       &                 &       \\ \hline
    SWaT 1              & 0.006       &                 &       \\ \hline
    SWaT 2              & 0.002       &                 &       \\ \hline
    SWaT 3              & 0.003       &                 &       \\ \hline
    SWaT 4              & 0.005       &                 &       \\ \hline
    SWaT 5              & 0.004       &                 &       \\ \hline
    SWaT 6              & 0.010       &                 &       \\ \hline
    ecoli               & 0.505       &                 &       \\ \hline
    cmc                 & 0.458       &                 &       \\ \hline
    lympho h            & 0.381       &                 &       \\ \hline
    wbc h               & 0.348       &                 &       \\ \hline
    \end{tabular}
    \caption{Optimal Thresholds for Various Datasets; Main Text(Contribution) : Automated Hyperparameter Tuning}
    \label{table:5}
\end{minipage}
\end{table*}

\end{document}